\documentclass{article}

\usepackage{PRIMEarxiv}

\usepackage[utf8]{inputenc} 
\usepackage[T1]{fontenc}    
\usepackage{hyperref}       
\usepackage{url}            
\usepackage{booktabs}       
\usepackage{amsfonts}       
\usepackage{nicefrac}       
\usepackage{microtype}      
\usepackage{lipsum}
\usepackage{fancyhdr}       
\usepackage{graphicx}       
\graphicspath{{media/}}     

\usepackage{soul}
\usepackage{url}
\usepackage{amssymb}
\usepackage[small]{caption}
\usepackage{graphicx}
\usepackage{amsmath}
\usepackage{amsthm}
\usepackage{algpseudocode,algorithm}
\usepackage{tikz}
\usepackage{enumitem}

\newtheorem{proposition}{Proposition}[section]
\newtheorem{theorem}[proposition]{Theorem}
\newtheorem{cor}[proposition]{Corollary}
\newtheorem{lemma}[proposition]{Lemma}

\theoremstyle{definition}
\newtheorem{example}[proposition]{Example}
\newtheorem{definition}[proposition]{Definition}

\algnewcommand\algorithmicswitch{\textbf{switch}}
\algnewcommand\algorithmiccase{\textbf{case}}
\algnewcommand\algorithmicassert{\texttt{assert}}
\algnewcommand\Assert[1]{\State \algorithmicassert(#1)}%
\algdef{SE}[SWITCH]{Switch}{EndSwitch}[1]{\algorithmicswitch\ #1\ \algorithmicdo}{\algorithmicend\ \algorithmicswitch}%
\algdef{SE}[CASE]{Case}{EndCase}[1]{\algorithmiccase\ #1:}{\algorithmicend\ \algorithmiccase}%
\algtext*{EndSwitch}%
\algtext*{EndCase}

\newcommand{\AF}{\mathcal{F}}
\newcommand{\A}{\mathcal{A}}
\newcommand{\Adj}{\mathbb{A}}

\newcommand{\D}{\mathcal{D}}
\newcommand{\Att}{\mathtt{Att}}

\def\diag{\operatorname{diag}}
\def\RR{\mathbb{R}}
\newcommand{\xto}[1]{\xrightarrow{#1}}
\newcommand{\OP}[1]{\operatorname{#1}}

\title{Analytical Solutions for the Inverse Problem within Gradual Semantics
}

\author{
  Nir Oren, Bruno Yun, Assaf Libman, Murilo S. Baptista\\
  University of Aberdeen \\
  Scotland \\
  \texttt{\{n.oren, bruno.yun, a.libman, murilo.baptista\}@abdn.ac.uk} \\
   \And
  Srdjan Vesic \\
  CRIL, Univ. Artois \\
  France\\
  \texttt{vesic@cril.fr} \\
}

\begin{document}
\maketitle

\begin{abstract}
Gradual semantics within abstract argumentation associate a numeric score with every argument in a system, which represents the level of acceptability of this argument, and from which a preference ordering over arguments can be derived. While some semantics operate over standard argumentation frameworks, many utilise a weighted framework, where a numeric initial weight is associated with each argument. Recent work has examined the inverse problem within gradual semantics. Rather than determining a preference ordering given an argumentation framework and a semantics, the inverse problem takes an argumentation framework, a gradual semantics, and a preference ordering as inputs, and identifies what weights are needed to over arguments in the framework to obtain the desired preference ordering. Existing work has attacked the inverse problem numerically, using a root finding algorithm (the bisection method) to identify appropriate initial weights. In this paper we demonstrate that for a class of gradual semantics, an analytical approach can be used to solve the inverse problem. Unlike the current state-of-the-art, such an analytic approach can rapidly find a solution, and is guaranteed to do so. In obtaining this result, we are able to prove several important properties which previous work had posed as conjectures.
\end{abstract}

\section{Introduction}

Standard approaches to abstract argumentation consider a set of atomic arguments and the interactions between them, encoding these as a graph. They then identify which sets of arguments are justified together by considering the inter-argument interactions \cite{baroni_introduction_2011}. Within the argumentation community, there has been increasing interest in so-called \emph{ranking-based} semantics. These latter semantics aim to identify a ranking over the arguments, with higher ranked arguments considered more justified (i.e. ``less attacked'') than arguments ranked lower (i.e. ``more attacked''). Many ranking-based semantics associate  numerical \emph{acceptability degrees} to all arguments within the system, with the final ranking depending on the associated numerical ordering. Furthermore, some ranking semantics compute the final acceptability degree of an argument based not only on the topology of the argumentation graph, but also based on some \emph{initial weight} assigned to the argument.

Our focus in this paper is on the \emph{inverse problem} found in ranking-based semantics. That is, rather than describing how an ordering over arguments can be obtained from some set of arguments and their associated properties, we ask \emph{what properties must be associated with arguments so as to derive some desired final acceptability ordering}. This problem was previously tackled by \cite{oren2022inverse}, who employed a numerical approach to solve it. In their work, the authors focused on three specific semantics, which also serve as the focus of this paper. Our core contribution involves describing an analytic approach to solving the inverse problem for these three semantics.
The advantages of pursuing such an analytic approach are twofold. First, it guarantees our ability to find a (unique) solution; and second, it is  efficient in finding this solution.

Apart from describing an analytic approach to the inverse problem, we make advance the state-of-the-art in several ways. First, building on ideas from \cite{Pu14,AMGOUD2022103607,yun20ranking} we describe a very general class of functions, which can be used to underpin gradual argumentation semantics, and demonstrate the existence of a unique fixed-point for such functions. We also demonstrate continuity for this class of functions. Focusing on the h-categorizer gradual semantics \cite{amgoud_acceptability_2017}, we prove that this semantics obeys monotonicity, a conjecture made in \cite{oren2022inverse} which is required for their numerical approach to operate. With these results, we are able to prove several other important properties hitherto not discussed in the literature.

The remainder of this paper is structured as follows. In the Section \ref{sec:background} we introduce the semantics for which we address the inverse problem. In Section \ref{sec:kernel}, we consider a generalisation of several gradual semantics, demonstrating that all semantics in this generalisation have a unique fixed point and are continuous. In Section \ref{sec:problem} we expand on the approach advanced in \cite{oren2022inverse}, demonstrating how it can applied to the more generalised semantics. Following this, Section \ref{sec:analitic-approach} details our analytic approach and proves some important properties. Section \ref{sec:conclusion} concludes and discusses potential future work.

\section{Background}\label{sec:background}

We begin by providing an overview of the three gradual semantics around which our work revolves. These semantics all operate over \emph{weighted argumentation frameworks} \cite{amgoud_acceptability_2017,coste-marquis_selecting_2012}.

\begin{definition}[WAF]
A weighted argumentation framework (WAF) is a triple $\AF=\langle \A,\D,w \rangle$, where $\A$ is a finite set of arguments, $\D \subseteq \A \times \A$ is a binary attack relation, and $w: \A \to [0,1]$ is a weighting function assigning an initial weight to each argument.
\end{definition}

We denote the set of attackers of an argument $a \in \A$ as $\Att(a)=\{b \in \A \mid (b,a) \in \D\}$.
While myriad ranking based-semantics have been described (see e.g., \cite{bonzon_comparative_2016}), we concentrate on three of the semantics described in \cite{AMGOUD2022103607} which allow for an initial weight to be assigned to an argument. This is because the inverse problem we focus on considers what initial weight needs to be assigned to obtain some preference ordering.

\begin{definition}[Gradual Semantics]
\label{def:grad_sem}
A gradual semantics $\sigma$ is a function that associates to each weighted argumentation graph $\AF=\langle \A,\D,w \rangle$, a scoring function $\sigma^\AF : \A \to [0,1]$ that provides an acceptability degree to each argument.
We consider the three semantics $\sigma_x$, for $x \in \{MB, CB, HC\}$, defined as follows.
\begin{itemize}
    
    \item The \emph{weighted max-based} semantics $\sigma_{MB}$ \cite{AMGOUD2022103607} is defined such that the acceptability degree of an argument $a\in \A$ is $\sigma_{MB}^\AF(a) = MB_\infty(a)$, where $MB_i(a)=\frac{w(a)}{1+\max\limits_{b \in \Att(a)} MB_{i-1}(b)}$ and for all $b \in \A, MB_0(b) = w(b)$
    .

    \item The \emph{weighted card-based} semantics $\sigma_{CB}$ \cite{AMGOUD2022103607} is defined such that the acceptability degree of an argument $a\in \A$ is $\sigma^\AF_{CB}(a) = CB_\infty(a)$ where
    $CB_i(a)=\frac{w(a)}{1+|\Att^*(a)|+ \frac{\sum\limits_{b \in \Att^*(a)} CB_{i-1}(b)}{|\Att^*(a)|}}$, for all $b \in \A, CB_0(b) = w(b)$, and $\Att^*(a) = \{ b \in \Att(a) \mid w(b)>0\}$ if $\Att^*(a) \neq \emptyset$ and $w(a)$ otherwise.
    
    \item The \emph{weighted h-categorizer} semantics $\sigma_{HC} $\cite{AMGOUD2022103607} is defined such that the acceptability degree of an argument $a \in \A$ is $\sigma_{HC}^\AF(a) = HC_\infty(a)$ where
    $HC_i(a)=\frac{w(a)}{1+\sum\limits_{b \in \Att(a)} HC_{i-1}(b)}$ and for all $b \in \A, HC_0(b) = w(b)$.    
    
\end{itemize}
\end{definition}

\begin{example} (taken from \cite{oren2022inverse}).  \label{ex:semantics}
Let $\AF = \langle \A, \D, w \rangle$ be a WAF, where $\A = \{ a_0, a_1, a_2, a_3 \}, \D = \{ (a_0, a_2),(a_1, a_1),(a_1, a_2),(a_2, a_2),$ $(a_3, a_2) \}$, $w(a_0) = 0.43, w(a_1) = 0.39, w(a_2) = 0.92$, and $w(a_3) = 0.3$. The WAF is represented in Figure \ref{fig:ex1} whereas the acceptability degrees and the associated rankings on arguments for the semantics of Definition \ref{def:grad_sem} are shown in Table \ref{tab:ex1}.

Note that we use the following notation for the ordering on arguments, where $a \trianglerighteq b$ denotes that $a$ is at least as preferred as $b$, $a \simeq b$ iff $ a \trianglerighteq b \wedge b \trianglerighteq a$, $a \triangleright b$ iff $a \trianglerighteq b \wedge a \not\trianglelefteq b$, and $a \trianglelefteq b$ iff $a \not \triangleright\ b$.

\begin{table}[t]
    \centering
    \renewcommand{\arraystretch}{1.3}
    \begin{tabular}{|c|c|c|c|c|c|}
    \hline
         & $a_0$& $a_1$& $a_2$& $a_3$& Argument ranking \\
         \hline
         $\sigma^\AF_{MB}$& 0.43 & 0.30 & 0.58 & 0.30 & $a_1 \simeq a_3 \triangleleft a_0 \triangleleft a_2$\\
         \hline
         $\sigma^\AF_{HC}$& 0.43 & 0.30 & 0.38 & 0.30 & $ a_1 \simeq a_3 \triangleleft a_2 \triangleleft a_0$\\
         \hline
         $\sigma^\AF_{CB}$& 0.43 & 0.18 & 0.17 & 0.30 & $ a_2 \triangleleft a_1 \triangleleft a_3 \triangleleft a_0$\\
         \hline
    \end{tabular}
    \caption{Acceptability degrees of the arguments from Figure \ref{fig:ex1}}
    \label{tab:ex1}
\end{table}

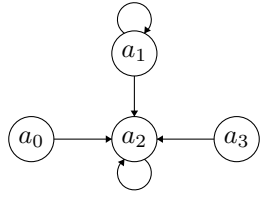
\begin{figure}
\centering
\begin{tikzpicture}[scale=0.1]
\tikzstyle{every node}+=[inner sep=0pt]
\draw [black] (18.5,-39.5) circle (3);
\draw (18.5,-39.5) node {$a_0$};
\draw [black] (32.2,-28.1) circle (3);
\draw (32.2,-28.1) node {$a_1$};
\draw [black] (32.2,-39.5) circle (3);
\draw (32.2,-39.5) node {$a_2$};
\draw [black] (45.8,-39.5) circle (3);
\draw (45.8,-39.5) node {$a_3$};

\draw [black] (21.5,-39.5) -- (29.2,-39.5);
\fill [black] (29.2,-39.5) -- (28.4,-39) -- (28.4,-40);
\draw [black] (30.877,-25.42) arc (234:-54:2.25);
\fill [black] (33.52,-25.42) -- (34.4,-25.07) -- (33.59,-24.48);
\draw [black] (32.2,-31.1) -- (32.2,-36.5);
\fill [black] (32.2,-36.5) -- (32.7,-35.7) -- (31.7,-35.7);
\draw [black] (33.523,-42.18) arc (54:-234:2.25);
\fill [black] (30.88,-42.18) -- (30,-42.53) -- (30.81,-43.12);
\draw [black] (42.8,-39.5) -- (35.2,-39.5);
\fill [black] (35.2,-39.5) -- (36,-40) -- (36,-39);
\end{tikzpicture}
\caption{Graphical representation of a WAF}
\label{fig:ex1}
\end{figure}

\end{example}

\section{Kernel-based Gradual Semantics} \label{sec:kernel}

Pu \cite{Pu14} demonstrated that a unique fixed-point exists for the h-categorizer semantics, i.e.\ in the case where all arguments have an initial weight of 1. This result was generalised to the three semantics described in Definition \ref{def:grad_sem} by Amgoud et al. \cite{AMGOUD2022103607}. In this section, we describe a more general class of semantics for which such a unique fixed-point is guaranteed to exist. Our core result is as follows.

\begin{theorem}
\label{th:fix_point}
We say that $\varphi_i(x):[0,1]^n \to [0, \infty)$, where $i=1, \dots,n$, is an argumentation kernel function if $\varphi_i$:
\begin{itemize}
    \item is continuous
    \item is monotonic, i.e., for all $x,y \in [0,1]^n$ where $x \preceq y$\footnote{We equip $\RR^n$ with the partial order $x \preceq y$ if $x_i \leq y_i$ for all $i=1, \dots,n$. We will also write $x \preceq 0$ to denote $x_i \leq 0$ for all $i$.}, $\varphi_i(x) \leq \varphi_i(y)$
    \item is homogeneous, i.e., $\varphi_i(tx) = t \varphi(x)$ for any $x \in p0,1]^n$ and every $0 \leq t \leq 1$
\end{itemize}

If $\Phi_w(x)_i=\frac{w_i}{1+\varphi_i(x)}$, then $\Phi_w:[0,1]^n \to [0,1]^n$ has a unique fixed point.
\end{theorem}

We can consider the class of all gradual semantics $\sigma$ that takes as input a WAF $\AF = \langle \{ a_1, a_2, \dots, a_n\}, \D, w \rangle$ and return a scoring function $\sigma^\AF$ such that for all $x \in [0,1]^n$, for all $a_i \in \A$, there exists $\Phi_w$ (as defined in Theorem \ref{th:fix_point}) such that $\Phi_w(x)_i = \sigma^\AF(a_i)$.
The previous theorem states that all the gradual semantics in the aforementioned class are guaranteed to ``converge'', including the previous three semantics of Definition \ref{def:grad_sem}.
Indeed, they are all in this class with $\varphi_i(x) = \max_{a_j \in \Att(a_i)} x_j$, 
$\varphi_i(x) = 1+|\Att^*(a_i)|+ (\sum_{a_j \in \Att^*(a_i)} x_j)/ |\Att^*(a_i)|$, and $\varphi_i(x) = \sum_{a_j \in \Att(a_i)} x_j$ respectively for $\sigma_{MB}, \sigma_{CB},$ and $\sigma_{HC}$. 
Of course, our result is more general and one could consider many other semantics in this class based on other argumentation kernel functions. 
Other examples of argumentation kernels include the geometric mean, i.e.\ $\varphi_i(x) = \sqrt[n]{b_i x_1 \cdots x_n}$ for some $b_i >0$, or the $L^p$-norm\footnote{We will also use the notation $\|x\|_p$ to denote the $L^p$ norm of $x$.
For $p=\infty$ we obtain the max-norm, i.e.\ 
$\|x\|_\infty = \max_i |x_i|.$}, i.e.\ $\varphi_i(x) = \sqrt[p]{\sum_j |x_j|^p}$ for some $1 \leq p  < \infty$, among others.
More generally, we can combine argumentation kernels together as the collection of argumentation kernel functions is: 
\begin{enumerate}
\item
\textit{closed under linear combination with non-negative coefficients}, i.e.\ if $\varphi_1,\dots,\varphi_k$ are argumentation kernels and $\lambda_1,\dots,\lambda_k \geq 0$ then $\lambda_1 \varphi_1 + \dots \lambda_k \varphi_k$ is an argumentation kernel.

\item
\textit{closed under geometric}, i.e.\ if $\varphi_1,\dots,\varphi_k$ are argumentation kernels then $\sqrt[k]{\varphi_i(x)\cdots \varphi_k(x)}$ is an argumentation kernel.

\item
\textit{closed under limits}, i.e.\ if $\varphi_k$ is a sequence of argumentation kernels such that the limit $\varphi :=\lim_k \varphi_k$ exists and continuous (e.g the convergence is uniform), then $\varphi$ is an argumentation kernel.
\end{enumerate}

The following theorem builds on the following result, adapted from \cite{Pu14}.

\begin{theorem}\label{th:pu-generalisation}
Let $X = [0,1]^n$, and let $f:X \to X$ be a function which
\begin{itemize}
    \item is order reversing, i.e., $x \preceq y$ implies $f(y) \preceq f(x)$
    \item there is some $0 < \alpha \leq 1$ such that for any $x \in [0,1]^n$ and any $0 \leq t \leq 1$
    $$f(tx) \preceq \frac{1}{t+\alpha(1-t)} f(x)$$
\end{itemize}
then $f$ has a unique fixed point $y \in [0,1]^n$. Furthermore, $y = \lim_{k \to \infty} x^{(k)}$ where $x^{(k)}$ is any sequence defined recursively by choosing $x^{(0)} \in X$ arbitrarily and $x^{(k+1)}=f(x^{(k)})$.
\end{theorem}

\begin{proof}
Define a sequence $u^{(n)}$ in $X$ by recursion, $u^{(0)}=0$ and $u^{(n+1)}=f(u^{(n)})$ for every $n \geq 0$.
Then $u^{(0)} \preceq u^{(1)}$ since $0 \in X$ is the minimum of $X$, and since $f$ satisfies the second item of Theorem \ref{th:pu-generalisation}, we get $u^{(0)} \preceq u^{(2)} \preceq u^{(1)}$.
By induction one easily shows
\begin{eqnarray}
\label{E:fixed points general:chains}
\nonumber & \text{(i)} & u^{(2k)} \preceq u^{(2k-1)} \text{ for all $k \geq 1$}, \\
\nonumber
& \text{(ii)} & u^{2k)} \preceq u^{(2k+2)}  \text{ for all $k \geq 0$}, \\
\nonumber
& \text{(iii)} & u^{2k+1)} \preceq u^{(2k-1)}  \text{ for all $k \geq 1$}.
\end{eqnarray}
Thus, the sequence $u^{(2k)}$ is increasing (in the sense that $u^{(2k)}_i$ is an increasing sequence in $\RR$ for every $i=1,\dots,n$) and bounded above by $1 \in X$, and $u^{(2k-1)}$ is decreasing and bounded below by $0 \in X$, so we define
\begin{align*}
    u^{(\OP{ev})} &= \sup_k u^{(2k)} = \lim_k u^{(2k)} \\ 
    u^{(\OP{odd})} &= \inf_k u^{(2k-1)} = \lim_k u^{(2k-1)}.
\end{align*}

It follows from (i) that $u^{(\OP{ev})} \preceq u^{(\OP{odd})}$.
Our next goal is to show that $u^{(\OP{ev})} = u^{(\OP{odd})}$.

For every $k \geq 1$ set
\[
\pi_k = \sup \, \{ t : t \cdot u^{(2k-1)} \preceq u^{(2k)}, \, 0 \leq t \leq 1\}.
\]
The set on the right is not empty since $0 \cdot u^{(2k-1)}=0 \preceq u^{(2k)}$, so
\[
0 \leq \pi_k \leq 1
\]
and by construction
\[
\pi_k \cdot u^{(2k-1)} \preceq u^{(2k)}.
\]
Since $f$ is order reversing,(ii) implies that for every $k \geq 1$

\begin{align*}
u^{(2k+1)} &= f(u^{(2k)}) \preceq \tfrac{1}{\pi_k+\alpha(1-\pi_k)} \cdot f(u^{(2k-1)}) \\
&= \tfrac{1}{\pi_k+\alpha(1-\pi_k)} \cdot u^{(2k)} 
\preceq \tfrac{1}{\pi_k+\alpha(1-\pi_k)} \cdot u^{(2k+2)}.
\end{align*}

Hence $(\pi_k+\alpha(1-\pi_k)) \cdot u^{(2k+1)} \preceq u^{(2k+2)}$, so by the definition of $\pi_{k+1}$
\[
\pi_k+\alpha(1-\pi_k) \leq \pi_{k+1}.
\]
It follows that $1-\pi_{k+1} \leq (1-\alpha)(1-\pi_k)$ for all $k \geq 1$  and therefore 
\[
1-\pi_{k+1} \leq (1-\pi_1) (1-\alpha)^k \xto{ \ k \to \infty \ } 0.
\]
But $\pi_k \leq 1$ for all $k$, so $\lim_k \pi_k =1$.

Consider some $\epsilon>0$.
Then $\pi_k >1-\epsilon$ for all $k \gg 0$.
Equation (i) implies
\[
(1-\epsilon) \cdot u^{(2k-1)} \preceq \pi_k \cdot u^{(2k-1)} \preceq  u^{(2k)} \preceq u^{(2k-1)}.
\]
Letting $k \to \infty$ this implies $(1-\epsilon) u^{(\OP{odd})} \preceq u^{(\OP{ev})} \preceq u^{(\OP{odd)}}$.
Since $\epsilon>0$ was arbitrary,
\[
u^{(\OP{ev)}}=u^{(\OP{odd})}
\]
as needed.
We will denote $u^*:=u^{(\OP{ev)}}=u^{(\OP{odd})}$.

Write $f^n \colon X \to X$ for the $n$-fold composition of $f$ with itself: $f^n=f \circ \cdots \circ f$. 
Since $0 \in X$ is minimal, $u^{(0)} \preceq x$ for any $x \in X$.
Since $f$ is order reversing, $u^{(0)} \preceq f(x) \preceq u^{(1)}$.
It then follows by induction that for all $k \geq 1$
\begin{eqnarray*}
&& f^{2k-1}(X) \subseteq [u^{(2k-2)},u^{(2k-1)}]_X \\
&& f^{2k}(X) \subseteq [u^{(2k)},u^{(2k-1)}]_X.
\end{eqnarray*}
We deduce that
\begin{align*}
\bigcap_{n \geq 1} f^n(X) &= 
\bigcap_{k \geq 1} f^{2k-1}(X) \cap f^{2k}(X) 
\\
& \subseteq 
\bigcap_{k \geq 1} [u^{(2k-2)},u^{(2k-1)}]_X \cap [u^{(2k)},u^{(2k-1)}]_X \\
  &= [u^{(\OP{ev})},u^{(\OP{odd}})]_X = \{u^*\}.
\end{align*}
Notice that $u^* \in [u^{(2k)},u^{(2k-1)}]_X$ (resp. $u^* \in [u^{(2k-2)},u^{(2k-1)}]_X$) for all $k \geq 1$, so since $f$ is order reversing $f(u^*) \in [u^{(2k)},u^{(2k+1)}]_X$ (resp. $u^* \in [u^{(2k)},u^{(2k-1)}]_X$) for all $k$.
Therefore $f(u^*) \in \{u^*\}$, so $u^*$ is a fixed point of $f$.

If $y \in X$ is a fixed point of $f$ then by induction $y \in f^n(X)$ for all $n$, so $y \in \cap_n f^n(X)=\{u^*\}$ and it follows that $y=u^*$.
Therefore $u^*$ is the unique fixed point of $f$.

Finally, choose some $x \in X$ and define a sequence by recursion $x^{(0)}=x$ and $x^{(n+1)}=f(x^{(n)})$.
Then $u^{(0)}=0 \preceq x$ and since $f$ is order reversing, $u^{(0)} \preceq x^{(1)} \preceq u^{(1)}$.
Then one proves by induction that $u^{(2k-2)} \preceq x^{(2k-1)} \preceq u^{(2k-1)}$ and $u^{(2k)} \preceq x^{(2k)} \preceq u^{(2k-1)}$ for all $k \geq 1$.
By the sandwich rule $\lim_n x^{(n)}=u^*$.

\end{proof}

We can now prove the following.
\begin{theorem}\label{th:satisfy_conditions}
A function $\Phi_w:X \to X$, based on some argumentation kernel functions, satisfies the conditions of Theorem \ref{th:pu-generalisation}. That is, $\Phi_w$ is order reversing and there exists $0 < \alpha \leq 1$ such that $\Phi_w(tx) \preceq \tfrac{1}{t+\alpha(1-t)} \Phi_w(x)$, for any $x \in X$ and any $0 \leq t \leq 1$.
\end{theorem}

\begin{proof}
Since $\varphi_i$ are monotonic, we get that if $x \preceq y$ then for any $i$
\[
\Phi_w(x)_i 
= \frac{w_i}{1+\varphi_i(x)} 
\geq \frac{w_i}{1+\varphi_i(y)} 
= \Phi_w(y).
\]
So $\Phi_w \colon X \to X$ is order reversing.
For the second condition, observe that the functions $\psi_i \colon X \to \RR$ defined by $\psi_i(x) = \tfrac{1}{1+\varphi_i(x)}$ are well defined, continuous and take value in the interval $(0,1]$.
Since $X$ is compact, set
\[
\alpha = \min_i \min_{x \in X} \tfrac{1}{1+\varphi_i(x)}.
\]
Then $0 < \alpha \leq 1$.
By construction, for any $x \in X$ and any $0 \leq t \leq 1$ and any $i=1,\dots,n$

\begin{align*}
\Phi_w(t x)_i
&= \frac{w_i}{1+\varphi_i(x)}  \\
&= \frac{w_i}{(1-t)+t(1+\varphi_i(x))} \\
&= \frac{\tfrac{w_i}{1+\varphi_i(x)}}{t+(1-t) \tfrac{1}{1+\varphi_i(x)}} \\
&\leq \frac{1}{t+\alpha(1-t)} \Phi_w(x)_i.
\end{align*}

Notice that the denominator is always positive.
It follows that $\Phi_w(t x)  \preceq \tfrac{1}{t+\alpha(1-t)} \Phi_w(x)$.
\end{proof}

This theorem, combined with Theorem \ref{th:pu-generalisation} directly proves Theorem \ref{th:fix_point}.

\begin{definition}\label{def:k}
Let $\varphi_i$ be some argumentation kernel functions, for $i = 1, \dots, n$ and $W = [0,1]^n$.
We define a function $h_\varphi \colon W \to X$ such that $h_\varphi((w_1, \dots, w_n))$ is the unique fixed point of $\Phi_w$. 
The image of $h_\varphi$ will be denoted by $H_\varphi$ and will also be referred to as the \textit{acceptability degree space} in later sections.
We also define a function $k_\varphi: X \to \RR^n$ such that $k_\varphi(x)_i =x_i(1+\varphi_i(x))$.
\end{definition}

We can demonstrate that any semantics which builds on an argumentation kernel is continuous through the following theorem.

\begin{theorem}
\label{th:homeomor}

The restriction of $k_\varphi$  to $H_\varphi$ gives a homeomorphism $k_\varphi \colon H_\varphi \to W$ whose inverse is $h_\varphi$.
In particular $h_\varphi \colon W \to X$ restricts to a homeomorphism onto $H_\varphi$.

\begin{proof}
Consider some $x \in H_\varphi$.
Then $x=h(w)$ for some $w = (w_1, \dots, w_n) \in W$.
By definition $x$ is a fixed point of $\Phi_w \colon X \to X$ so $x_i=\tfrac{w_i}{1+\varphi_i(x)}$ for all $i=1,\dots,n$, which is equivalent to $w_i=x_i(1+\varphi_i(x))$.
In other words, $k_\varphi(x)=w$.
We see that $k_\varphi$ restricts to a function $k_\varphi \colon H_\varphi \to W$ and moreover $k_\varphi(h_\varphi(w))=w$ for all $w \in W$. Observe thus that $k_\varphi \colon H_\varphi \to W$ is surjective.

Suppose that $x,x' \in H_\varphi$ are such that $k_\varphi(x)=k_\varphi(x')$.
That is, $k_\varphi(x)=k_\varphi(x')=w \in W$.
But $k_\varphi(x)=w$ means that $w_i=x_i(1+\varphi_i(x))$ so $x$ is a fixed point of $\Phi_w \colon X \to X$.
Similarly, $k_\varphi(x')=w$ means that $x'$ is a fixed point of $\Phi_w$.
But $\Phi_w$ has a unique fixed point in $X$, so $x=x'$.
It follows that $k_\varphi \colon H_\varphi \to W$ is injective.
Thus, $k_\varphi$ is bijective, and $k_\varphi \circ h_\varphi$ yielding $w$  implies that $k_\varphi^{-1}=h_\varphi$.

Next, we claim that $H_\varphi$ is a closed subset of $X$, hence it is compact.
To see this, let $x^{(i)} \in H_\varphi$ be a convergent sequence in $X$ with limit $y$.
Set $w^{(i)}=k_\varphi(x^{(i)})$.
Then $w^{(i)}$ in $W$ and since $k_\varphi$ is continuous $\lim_i w^{(i)} = \lim_i k_\varphi(x^{(i)})=k_\varphi(y)$. 
Since $W$ is closed in $\RR^n$ we deduce that $k_\varphi(y) \in W$.
But since $h_\varphi$ is the inverse of $k_\varphi \colon H_\varphi \to W$ this implies that $y=h_\varphi(k_\varphi(y) \in H_\varphi$.
This shows that $H_\varphi$ is closed, as needed.

Thus, $H_\varphi$ is a compact subset of $\RR^n$ and $k_\varphi$ is a bijective continuous function between compact metric spaces.
It is therefore a homeomorphism, and consequently so is $h_\varphi \colon W_\varphi \to H$.
\end{proof}
\end{theorem}

The previous theorem is crucial as it shows that for any gradual semantics based on argumentation kernel functions and any arbitrary WAF, one cannot obtain the same scoring function with two different weighting functions on arguments.
The continuity is also an important result which will be used during the bisection method to solve the inverse problem, as per \cite{oren2022inverse}.

\section{The Inverse Problem}\label{sec:problem}

Recall that our goal is to find a set of initial weights which, when applied to a specific argumentation framework under the chosen semantics will result in a desired preference ordering,  derived from the numeric acceptance degree computed for each argument.

The approach described in \cite{oren2022inverse} makes use of two phases to solve the inverse problem. In the first phase, a target acceptance degree is computed for each argument. In the second phase, a numerical method (the bisection method)  is used to find the initial weights which lead to this acceptance degree. Since the bisection method is designed to find the zeros of a function with only one variable, and since changing the initial weight of one argument can affect the acceptance degree of other arguments, repeated applications of the bisection method are often necessary. \cite{oren2022inverse} identify several strategies for selecting the argument to which the bisection method should be applied, and demonstrate that selecting the argument whose current acceptance degree is furthest away from its target acceptance degree works well in practice.

In this work we will provide an alternative approach to calculate the initial weights to achieve a desired acceptance degree. We therefore recall how \cite{oren2022inverse} identifies target acceptance degrees for solving the inverse problem. We begin by noting that the desired preference orderings can be represented as a sequence of non-empty sets $[Ar_0, \ldots, Ar_n]$ which partition the set of arguments $\A$ such that for any $a,b \in Ar_i$, $0 \leq i \leq n$, $a \simeq b$, and for any $a \in Ar_i, b \in Ar_j$ where $0 \leq i < j \leq n$, $a \triangleright b$.

Algorithm \ref{alg:mub}, taken from \cite{oren2022inverse}, associates a \emph{minimal upper bound} with every argument in the system. This minimal upper bound --- in effect --- identifies an achievable final degree for an argument by assuming that all of an argument's attackers have some large value.

If we consider argumentation kernel functions $\varphi_i(x)$ instead, what Algorithm \ref{alg:mub} does is identify a maximum value for $\varphi_i(x)$ in terms of other arguments in the system, and set the minimal upper bound based on this value. The main part of the algorithm can therefore be more generally rewritten for our larger class of semantics as shown in Algorithm \ref{alg:mub2}.

\begin{algorithm}[t]
\begin{algorithmic}
\Function{ComputeBounds}{$[],\_,\_$}
\State \Return $\{\}$
\EndFunction\medskip
\Function{ComputeBounds}{$[Ar_0, \ldots, Ar_n],max,\sigma$}
\Switch{$\sigma$}
\Case{$\sigma_{MB}$} 
  $min \gets max/(1+max+\zeta)$
\EndCase
\Case{$\sigma_{HC}$}
  $min \gets max/(1+\max\limits_{a \in Ar_0} |\Att(a)|+\zeta)$
\EndCase  
\Case{$\sigma_{CB}$}
  $min \gets max/(2+\max\limits_{a \in Ar_0} |\Att(a)|+\zeta)$
\EndCase
\EndSwitch
\State \Return $\{(Ar_0,min)\} \cup$ \Call{ComputeBounds}{$[Ar_1, \ldots, Ar_n],min, \sigma$}
\EndFunction
\medskip

\Function{ComputeBounds}{$[Ar_0, \ldots, Ar_n], \sigma$}

\Return \Call{ComputeBounds}{$[Ar_0, \ldots, Ar_n],1,\sigma$}
\EndFunction
\end{algorithmic}
\caption{Computing arguments' minimal upper bounds.} \label{alg:mub}
\end{algorithm}

\begin{algorithm}[t]
\begin{algorithmic}
\Function{ComputeBounds}{$[Ar_0, \ldots, Ar_n],max,\sigma$}
\State Let $\varphi_i$ be the argumentation kernel of $\sigma$
\State Let $m$ be the maximal value of $\varphi_i$ when evaluated for $i$ over all arguments in $Ar_0$.
\State $min \gets max/(1+m+\zeta)$
\State \Return $\{(Ar_0,min)\} \cup$ \Call{ComputeBounds}{$[Ar_1, \ldots, Ar_n],min, \sigma$}
\EndFunction
\end{algorithmic}
\caption{\label{alg:mub2}More general computation of arguments' minimal upper bound using an argumentation kernel.}
\end{algorithm}

We now depart from \cite{oren2022inverse}; whereas they used the bisection method to identify initial weights for arguments which achieve the desired acceptability degrees, we consider an analytic approach for doing so.

\section{Computing Initial Weights from Acceptability Degrees}\label{sec:analitic-approach}

In this section we identify a vector representation for each semantics which allows us to compute appropriate initial weights. 
Namely, given an arbitrary WAF $\AF = \langle \A, \D, w \rangle$ such that $\A = \{a_1, a_2, \dots, a_n\}, \overrightarrow{w}$ is the column vector $(w(a_1), w(a_2), \dots, w(a_n))^T$, and $\mathbb{A}$ is the adjacency matrix where $\mathbb{A}_{ij} = 1$ iff $(a_j, a_i) \in \D$ and 0 otherwise. 
Note that $\overrightarrow{1}$ is a column vector containing $1$s of length equal to $|\A|$.

For each semantics $\sigma_X$, for $X \in \{MB, CB, HC\}$, we denote by $\overrightarrow{X}_\infty$, the column vector $(\sigma^\AF_X(a_1), \sigma^\AF_X(a_2), \dots, \sigma^\AF_X(a_n))^T$ containing the acceptability degrees of all arguments.
We now consider each semantics in turn.

\subsection{H-categorizer semantics}

As we will do for all other semantics, we examine the equation for the weighted h-categorizer semantics once convergence has taken place. It holds that:

\[
\overrightarrow{HC}_\infty=\frac{\overrightarrow{w}}{\overrightarrow{1}+\mathbb{A}\overrightarrow{HC}_\infty}
\]


Here, the division occurs in an element-wise manner.
We can rewrite this equation as:

\[
\overrightarrow{HC}_\infty+\mathbb{M}\mathbb{A}\overrightarrow{HC}_\infty = \overrightarrow{w}
\]

Where $\mathbb{M}$ is the diagonal matrix such that $\mathbb{M}_{ii} = \sigma^\AF_{HC}(a_i)$, for $1 \leq i \leq n$, and $0$ otherwise.
Since we know the values of $\overrightarrow{HC}_\infty$ and $\mathbb{M}$ (the minimum upper bounds), as well as $\mathbb{A}$ (the adjacency matrix obtained from the structure of our argumentation graph), we can therefore compute $\overrightarrow{w}$ (the initial weights) directly from the above equation.

\subsection{Card-Based Semantics}
Following the same process as above, at convergence, the weighted card-based semantics can be written as:

\[
\overrightarrow{CB}_\infty+\mathbb{D}\overrightarrow{CB}_\infty+\mathbb{D}^{-1}\mathbb{M}\mathbb{A}\overrightarrow{CB}_\infty=\overrightarrow{w}
\]

Here $\mathbb{D}$ is the diagonal matrix such that $\mathbb{D}_{ii} = |\Att^*(a_i)|$, for $1 \leq i \leq n $, and $0$ otherwise.
Note that this matrix will contain only 0s in  row $i$ if argument $a_i$ is unattacked or all its attackers have an initial weight of 0. However from the definition of the semantics, the acceptability degree of such an argument should be equal to its initial weight. In cases where the diagonal element is non-zero, we can simply solve the equation by performing an index-wise calculation.  Notice that  $\mathbb{D}^{-1}$ is also a diagonal matrix where the elements are the reciprocal of the diagonal elements of  $\mathbb{D}$. 

\subsection{Weighted Max-Based Semantics}

The $\max$ operation present in these semantics can be written in a vector format as:


\[
\overrightarrow{w}=\overrightarrow{MB}_\infty+\mathbb{M}\max\{\mathbb{A}\mathbb{O}\}, 
\]
\noindent

Where  $\mathbb{O}$ is a square matrix whose columns are all equal to the vector $\overrightarrow{MB}_\infty$, and  $\max\{\mathbb{A}\mathbb{O}\}$ takes the largest element from each row of $\mathbb{A}\mathbb{O}$ to form a column vector. 

\subsection{Generalisation}

We observe that in general, the function $k_\varphi$, as per Definition \ref{def:k}, can be used to compute the initial weights by passing the desired final degree as input (i.e., given an acceptability degree vector $h \in H_\varphi$, we can compute $k_\varphi(h)$ to obtain our initial weights analytically).

\begin{example}
Consider the argumentation system shown in Figure \ref{fig:ex1}, and assume we wish to obtain the preference ordering $a_0 \triangleright a_1 \simeq a_3 \triangleright a_2$.
Table \ref{tab:acceptDeg} show, for each semantics, the acceptability degrees as computed by Algorithm \ref{alg:mub} (assuming $\zeta=1$). Table \ref{tab:initialWeights} then shows the initial weights necessary to achieve these acceptability degrees.

\begin{table}[t]
    \centering
    \renewcommand{\arraystretch}{1.3}
    \begin{tabular}{|c|c|c|c|c|}
    \hline
    & $a_0$& $a_1$& $a_2$& $a_3$ \\
    \hline
    $\sigma^\mathcal{F}_{HC}$ & 0.5 & 0.167 & 0.0278 & 0.167 \\
    \hline
    $\sigma^\mathcal{F}_{MB}$ & 0.333 & 0.111 & 0.037 & 0.111 \\
    \hline
    $\sigma^\mathcal{F}_{CB}$ & 0.333 & 0.0833 & 0.012 & 0.083 \\
    \hline
    \end{tabular}
    \caption{\label{tab:acceptDeg}Acceptability degrees computed by Algorithm \ref{alg:mub} ($\zeta=1$)}
\end{table}

\begin{table}[t]
    \centering
    \renewcommand{\arraystretch}{1.3}
    \begin{tabular}{|c|c|c|c|c|}
    \hline
    & $w(a_0)$& $w(a_1)$& $w(a_2)$& $w(a_3)$ \\
    \hline
    $\sigma_{HC}$ & 0.5 & 0.194 & 0.052 & 0.167 \\
    \hline
    $\sigma_{MB}$ & 0.333 & 0.123 & 0.0494 & 0.111 \\
    \hline
    $\sigma_{CB}$ & 0.333 & 0.174 & 0.061 & 0.083 \\
    \hline
    \end{tabular}
    \caption{\label{tab:initialWeights}Initial weights computed for each semantics}
\end{table}

\end{example}

\subsection{Properties}
Our results from the previous sections address the Weighting Validity conjecture described in \cite{oren2022inverse}, which states that initial weights can be found such that the acceptability degrees of each argument is equal to the corresponding the minimum upper bound generated by Algorithm \ref{alg:mub}.

We can also easily see that decreasing (resp. increasing) the acceptability degree of some arguments without increasing (resp. decreasing) the degree of any other arguments can only come about by only decreasing (resp. increasing) some initial weights while not increasing (resp. decreasing) any initial weights.

\begin{proposition}
Let us consider two WAFs $\AF = \langle \{ a_1, a_2, \dots, a_n\}, \D ,w \rangle$ and $\AF' = \langle \{ a_1, a_2, \dots, a_n\}, \D,w' \rangle$ and the two corresponding acceptability degree vectors $\overrightarrow{X}
_\infty,\overrightarrow{X'}_\infty$, for $X \in \{ HC, MB, CB\}$ such that for any $1 \leq i,j\leq n, (\overrightarrow{X}_{\infty })_i \leq (\overrightarrow{X'}_{\infty})_j$ $(resp. (\overrightarrow{X}_{\infty })_i \geq (\overrightarrow{X'}_{\infty})_j)$, it is the case that for any $1 \leq i,j\leq n, \overrightarrow{w}_i \leq \overrightarrow{w'}_j$ (resp. $\overrightarrow{w}_i \geq \overrightarrow{w'}_j$).
\end{proposition}

It trivially follows that for $X \in \{ HC, MB, CB\}$, if $\overrightarrow{X}_\infty$ is a valid acceptability degree vector, so is any $\overrightarrow{X'}_\infty$ such that for any $1 \leq i \leq n, 0 \leq (\overrightarrow{X'}_\infty)_i \leq (\overrightarrow{X}_{\infty})_i$.
Similarly, increasing the acceptability degree of one argument will increase (or not affect) all initial weights, but we must then ensure that no initial weights exceeds 1. 
For example, if $\AF = \langle \{a \}, \{ (a,a)\}, w \rangle$ with $w(a)=1$ then $\sigma^\AF_{HC}(a) \approx 0.62$. In this case, it is not possible to increase the acceptability degree of $a$ without the initial weight of $a$ going above $1$.


The result of Theorem \ref{th:homeomor} demonstrates that  for any given unweighted argumentation graph, the acceptability degree space (i.e., the set of all valid acceptability degree vectors $H_\varphi$) for $\sigma \in \{\sigma_{MB}, \sigma_{CB}, \sigma_{HC}\}$, will be continuous. Namely, we can always find a sequence of acceptability degree vectors that ``links'' two acceptability degree vectors. This builds on the result of \cite{oren2022inverse}, which demonstrated continuity, but only in the interval $[0,\infty)$.
However, this result does not mean that the acceptability degree space will be convex, i.e.\ if $\overrightarrow{X_\infty^1}$ and $\overrightarrow{X_\infty^2}$ are two valid acceptability degree vectors, then $\alpha \overrightarrow{X_\infty^1} + (1-\alpha) \overrightarrow{X_\infty^2}$, for $\alpha \in [0,1]$, will not always be a valid acceptability degree vector.
For example, if $\langle \{a_1, a_2, a_3\}, \{(a_1,a_2), (a_2,a_1), (a_1,a_3), (a_2,a_3) \} \rangle$ is an unweighted argumentation graph, we can find weighting functions such that both $(0,0,1)$ and $(1,0,0.5)$ are valid acceptability degree vectors w.r.t. $\sigma_{HC}$ but there is no $\alpha \in ]0,1[$ such that $\alpha(0,0,1) + (1-\alpha)(1,0,0.5)$ is a valid acceptability degree vector.

In the next proposition, we show that having self-attacking arguments prevent us from achieving some acceptability degree vectors.

\begin{proposition}
Given an unweighted argumentation graph $\langle \{a_1, a_2,\dots a_n\}, \D \rangle$, for all for $X \in \{HC, MB,CB\}$, $i \in \{1, \dots, n\}$, there exists $\overrightarrow{X_\infty}$ in the acceptability degree space such that $[\overrightarrow{X_\infty}]_i = 1$ iff there is no $a \in \{a_1, \dots, a_n\}$ such that $(a,a) \in \D$.
\end{proposition}

\begin{proof}
This proof is split into two parts. First, assume that there is a self-attacking argument $a_i \in \A$, it is easy to show that the maximum acceptability degree of $a_i$ is $(-1+ \sqrt{5})/2 \simeq 0.618$, for $\sigma_{HC}$ and $\sigma_{MB}$, and $-1 + \sqrt{2} \simeq 0.41$ for $\sigma_{CB}$, thus it is not possible to find an acceptability degree vector $\overrightarrow{X_\infty}$ such that $[\overrightarrow{X_\infty}]_i = 1$, for $X \in \{HC, MB,CB\}$.
Now, assume that there are no self-attacking arguments in $\A$, then for all $X \in \{HC, MB,CB\}$, $i \in \{1, \dots, n\},$ we can create an acceptability degree vector $\overrightarrow{X_\infty}$ such that $[\overrightarrow{X_\infty}]_i = 1$ by putting the weight of $a_i$ to 1 and the initial weights of all other arguments to $0$.
\end{proof}

\begin{definition}
Given two unweighted argumentation graphs $\langle \A, \D \rangle$ and $\langle \A', \D'\rangle$, we say that $f$ is an isomorphism from $\A$ to $\A'$ if $(a,b) \in \D$ iff $(f(a), f(b)) \in \D'$.
\end{definition}

The next proposition shows that isomorphic arguments will induce a symmetry effect in the acceptability degree space.

\begin{proposition}
Given an unweighted argumentation graph $\langle \A, \D \rangle, \A = \{a_1, \dots, a_n\}$ and $a_i, a_j \in \A$ such that there is an isomorphism $f$ from $\A$ to $\A$ and $f(a_i) = a_j$, it holds that for all $X \in \{ HC, MB, CB\},$ $\overrightarrow{X_\infty}$ is in the acceptability degree space iff $([\overrightarrow{X_\infty}]_1, \dots, [\overrightarrow{X_\infty}]_{i-1}, [\overrightarrow{X_\infty}]_j, [\overrightarrow{X_\infty}]_{i+1}, \dots,$ $[\overrightarrow{X_\infty}]_{j-1}, [\overrightarrow{X_\infty}]_i, [\overrightarrow{X_\infty}]_{j+1}, \dots,$ $[\overrightarrow{X_\infty}]_n)^T$ is in the acceptability degree space.
\end{proposition}

The previous proposition follows directly from the anonymity principle which is satisfied by all three semantics studied in the paper \cite{AMGOUD2022103607}. 
One important observation is that, apart from the structure of the graph, the semantics chosen also plays an important role in the shape of the acceptability degree space.
For instance, if we consider a complete graph of size $n$ with self-attacking arguments, the acceptability degree space for $\sigma_{MB}$ is $\{ (v_1, v_2, \dots, v_n)^T \mid \forall i \in \{1, \dots,n\}, 0 \leq v_i \leq  (-1 + \sqrt{5})/2 \}$ while for $\sigma_{CB}$ it is $\{ (v_1, v_2, \dots, v_n)^T \mid \forall i \in \{ 1, \dots,n \}, 0 \leq v_i \leq (-1 + \sqrt{5})/2$ and $v_i(1+ \sum_{j=1}^n v_j) \in [0,1]\}$.
We provide an graphical representation of the acceptability degree space for $\sigma_{HC}$ on an example in Figure \ref{fig:hcat_degree_space}.
It is clear here that the acceptability degree space for $\sigma_{CB}$ is a subset of the acceptability degree space for $\sigma_{MB}$ for all $n > 0$. In the next proposition, we show that this inclusion holds in the general case.

\begin{proposition}
Given an arbitrary unweighted argumentation graph $\langle \A, \D \rangle, \A = \{a_1, \dots, a_n\}$, the acceptability degree space for $\sigma_{HC}$ is a subset of the acceptability degree space for $\sigma_{MB}$.
\end{proposition}

\begin{proof}
Let us consider an arbitrary acceptability degree vector $\overrightarrow{HC}_\infty = (v_1, \dots, v_n)$ such that $v_i \in [0,1]$ for all $1 \leq i \leq n$. We show that this vector is in the acceptability degree space for $\sigma_{MB}$. By definition, we know that for all $i \in \{1, \dots, n \}, 0 \leq v_i (1+ \sum_{a_j \in \Att(a_i)}v_j) \leq 1$. It follows that for all $i \in \{1, \dots, n \}, 0 \leq v_i( 1 + \max_{a_j \in \Att(a_i)} v_j) \leq 1$ and thus, $\overrightarrow{HC}_\infty$ is in the acceptability degree space for $\sigma_{MB}$.
\end{proof}

\begin{figure}
    \centering
    \includegraphics[width=5cm]{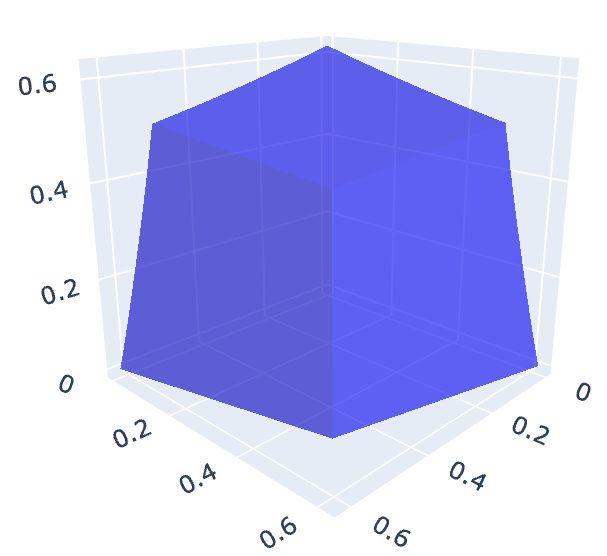}
    \caption{Representation of the acceptability degree space for $\sigma_{HC}$ on a complete argumentation graph with 3 arguments (in blue). Each axis represents the acceptability degree of one argument.}
    \label{fig:hcat_degree_space}
\end{figure}

The final property we consider is monotonicity, which --- together with continuity --- is required for the bisection method of \cite{oren2022inverse} to operate. Due to space constraints, we only prove this result for the weighted h-categoriser semantics. This means that we fix our argumentation kernel function of Section \ref{sec:kernel} to be $\varphi_i(x)=\|(\Adj x)_i\|_1=\sum a_{i,j}x_j$.

We will equip $\mathbb{R}^n$ with the $L^\infty$ norm.
Recall that there is an induced {\em operator norm} \cite{Alt} on the set of all linear transformations of $\RR^n$ which we denote as $\OP{End}(\RR^n) \cong \OP{Mat}_{n \times n}(\RR)$ 
where for any matrix $A$ we set this operator norm as 
\begin{align*}
\| A\| &= \sup \{ \frac{\|Ax\|_\infty}{\| x\|_\infty} : 0 \neq x \in \RR^n \} 
\\
& =
\max \{ \|Ax\|_\infty : x \in \RR^n, \|x\|_\infty=1\}  \\
& =
\max \{ \frac{\|Ax\|_\infty}{\| x\|_\infty} : 0 \neq x \in \RR^n\}.
\end{align*}
It is well known (and easy to check) that this defines a norm on $\OP{End(\RR^n)}$.
Since all norms on $\RR^k$ are equivalent, the operator norm makes $\OP{End}(\RR^n)$ a complete normed space.
By definition for any $x \in \RR^n$
\[
\| Ax\|_\infty \leq \| A\| \cdot \|x\|_\infty
\]
Moreover, the operator norm is multiplicative, namely
\[
\| AB\| \leq \|A \| \cdot \|B \|
\]
This is a standard construction, see for example \cite[Section 5.2 and Theorem 5.3 and Remark 5.4(1)]{Alt}.

\begin{lemma}\label{L:positive determinant zero diagonal}
Consider a matrix $A \in \OP{Mat}_{n \times n}(\RR)$ such that for every $i=1,\dots n$
\begin{enumerate}[label=(\alph*),leftmargin=*]
\item $a_{i,i}=0$ 

\item $\sum_j |a_{i,j}| <1$.
\end{enumerate}
Then $\det(I+A)>0$.
\end{lemma}

\begin{proof}
First we show that any matrix $A$ with these properties is invertible.
For any $x \in \RR^n$ with $\|x \|_\infty =1$ we get
\begin{align*}
\|Ax\|_\infty 
& = \max_i | Ax_i| 
= \max_i | \sum_j a_{i,j} x_j| \\
& \leq \max_i \sum_j |a_{i,j}| \cdot |x_j| 
\leq \max_i \sum_j |a_{i,j}| 
\end{align*}
It follows that $\|A\| \leq \max_i \sum_j |a_{i,j}|  <1$.
Since $\OP{End}(\RR^n)$ with the operators norm is complete, $B=\sum_{k=0}^\infty (-1)^k A^k$ converges and $B(I+A)=I$.
It follows that $I+A$ is invertible.

Let $U \subseteq \OP{Mat}_{n \times n}(\RR)$ be the (open) subset of all matrices $A$ satisfying the conditions in the lemma.
Clearly $0 \in U$.
Moreover $U$ is path connected since for any $A \in U$ the function $\lambda \colon [0,1] \xto{t \mapsto tA} \OP{Mat}_{n \times n}(\RR)$ gives a path in $U$ from $0$ to $A$.
The function $f \colon U \to \RR$ defined by $f(A)=\det(I+A)$ is clearly continuous.
Since $I+A$ is invertible, $f(A) \neq 0$ for all $A \in U$.
Since $f(0)=1>0$, the intermediate value theorem implies that $f(A)>0$ for all $A \in U$.
\end{proof}

\begin{cor}\label{C:det positive diagonal exceeds row sums}
Let $M \in \OP{Mat}_{n \times n}(\RR)$ be a matrix which satisfies the conditions
\begin{enumerate}[leftmargin=*]
\item $m_{i,j} \geq 0$ for all $i,j$
\item $m_{i,i} > \sum_{j \neq i} m_{i,j}$ for every $i=1,\dots,n$.
\end{enumerate}
Then $\det(M)>0$.
\end{cor}

\begin{proof}
The conditions imply $m_{i,i}>0$ for all $i=1,\dots,n$.
By inspection 
\[
\OP{diag}(m_{1,1}^{-1},\dots,m_{n,n}^{-1}) \cdot M = I +A
\] 
where $A$ is of the form in Lemma \ref{L:positive determinant zero diagonal}.
Then $\det(M)=\det(I+A) \cdot \prod_i m_{i,i} >0$. 
\end{proof}

\begin{figure*}[h!]
\[
M = \begin{bmatrix}
1+ a_{1,1} x_1 + \sum_{j} a_{1,j} x_j    &  a_{1,2}x_1      & a_{1,3}x_1   & \cdots & a_{1,n}x_1 \\
a_{2,1}x_2  &   1+ a_{2,2} x_2 + \sum_{j} a_{2,j} x_j    &  a_{2,3}x_2      &  \cdots & a_{2,n}x_2 \\
\vdots     &                                               & \ddots         &    \vdots \\
a_{n,1}x_n  &   a_{n,2} x_n   & a_{n,3}x_n                                  & \cdots         & 1+ a_{n,n} x_n + \sum_{j} a_{n,j} x_j
\end{bmatrix}
\]
    \caption{The matrix of the $M$ with entries $\tfrac{\partial h_i}{\partial x_\ell}$}
    \label{fig:matrix}
\end{figure*}

\begin{cor}\label{C:det positive an positive diagonal}
Let $A$ be an $n \times n$ matrix with $0 \leq a_{i,j} \leq 1$.
Let $x \in \RR^n$ be a column vector with $0 \leq x_i \leq 1$.
Let $\delta_1,\dots,\delta_n>0$.
Let $M$ be the matrix given by
\[
m_{i,j} = \left\{
\begin{array}{ll}
\delta_i + \sum_k a_{i,k}x_k & \text{if $i=j$} \\
a_{i,j}x_i                   & \text{if $i \neq j$}
\end{array}
\right.
\]
Then $\det(M)>0$ and moreover the diagonal entries of $M^{-1}$ are positive, i.e $(M^{-1})_{i,i}>0$ for all $i$.
\end{cor}

\begin{proof}
First, we show that all matrices $M$ of this form have $\det(M)>0$.
Let $M'$ be the matrix with $m'_{i,i}=m_{i,i}$ and $m'_{i,j}=a_{i,j}x_j$.
Observe that $m'_{i,j} \geq 0$ and that for any $i$
\[
m'_{i,i} = \delta_i + \sum_j a_{i,j}x_j  > \sum_{j \neq i} a_{i,j}x_j = \sum_{j \neq i} m'_{i,j}.
\]
By Corollary \ref{C:det positive diagonal exceeds row sums} $\det(M')>0$.

Let $S_n$ denote a permutation of $n$ symbols.
 For any $\sigma \in S_n$, let $\OP{fix}(\sigma)$ denote those symbols which the permutation does not shift (i.e., the fixed points of $\sigma$) and $\OP{supp}(\sigma)$ those symbols which are permuted (i.e., its support). Clearly,  $\sigma$ induces a permutation of $\OP{supp}(\sigma)$.
Therefore 
\begin{align*}
\prod_{i \in \OP{supp}(\sigma)} m_{i,\sigma(i)} & =
\prod_{i \in \OP{supp}(\sigma)} a_{i,\sigma(i)} x_i \\
& =
\prod_{i \in \OP{supp}(\sigma)} a_{i,\sigma(i)} \cdot \prod_{i \in \OP{supp}(\sigma)} x_i 
\\
& =
\prod_{i \in \OP{supp}(\sigma)} a_{i,\sigma(i)} \cdot \prod_{i \in \OP{supp}(\sigma)} x_{\sigma(i)} \\
& =
\prod_{i \in \OP{supp}(\sigma)} a_{i,\sigma(i)} x_{\sigma(i)} =
\prod_{i \in \OP{supp}(\sigma)} m'_{i,\sigma(i)}.
\end{align*}
We can now compute
\begin{align*}
\det(M) &= 
\sum_{\sigma \in S_n} (-1)^\sigma \prod_i m_{i,\sigma(i)} \\
& = \sum_{\sigma \in S_n} (-1)^\sigma \prod_{i \in \OP{fix}(\sigma)} m_{i,i} \cdot \prod_{i \in \OP{supp}(\sigma)}  m_{i,\sigma(i)} 
\\
&= \sum_{\sigma \in S_n} (-1)^\sigma \prod_{i \in \OP{fix}(\sigma)} m'_{i,i} \cdot \prod_{i \in \OP{supp}(\sigma)}  m'_{i,\sigma(i)} \\
& = \sum_{\sigma \in S_n} (-1)^\sigma \prod_i m'_{i,\sigma(i)}  \\
& = \det(M') >0.
\end{align*}
We have shown that all matrices $M$ of the form in the statement (for any $n$) have positive determinant and in particular they are invertible.
Let $M[i,j]$ denote the $(i,j)$-minor of $M$, namely the $(n-1) \times (n-1)$ matrix formed by deleting the $i$th row and $j$th column of $M$.
A consequence of Cramer's rule
\[
(M^{-1})_{i,i} = \frac{\det(M[i,i])}{\det(M)}.
\]
Notice that $M[i,i]$ is a matrix of the form in the statement (with $A$ replaced with $A'=A[i,i]$ and $x$ replaced with $x'=(x_1,\dots,\widehat{x_i},\dots,x_n)$ and $\delta_i$ replaced with some $\delta_i'>\delta_i$).
So $\det(M[i,i])>0$ and as a consequence $(M^{-1})_{i,i}>0$.
\end{proof}

Let $\Adj$ denote the adjacency matrix of a directed WAF $\AF$. We can now apply these results to the argumentation kernel functions for $\sigma_{HC}$.

\begin{theorem}
Let $h_\varphi \colon W \to H_\varphi$ and $k_\varphi \colon H_\varphi \to W$ be associated with the argumentation kernel functions $\varphi_i(x)=\|(\Adj x)_i\|_1=\sum a_{i,j}x_j$.
Then $h_\varphi$ and $k_\varphi$ are $C^\infty$ (differentiable infinitely many often) and moreover, for every $w \in W$ and every $i$
\[
\frac{\partial h_i}{\partial w_i}(w) >0.
\]
In particular $h_\varphi$ is increasing in each fibre, i.e if $w,w' \in W$ are such that $w'-w=(0,\dots,0,\epsilon_i,0,\dots,0)$ for some $\epsilon_i>0$ then $h_\varphi(w) < h_\varphi(w')$.
\end{theorem}

\begin{proof}

The function $k_\varphi$ is the restriction to $H_\varphi$ of the $C^\infty$ function $k_\varphi \colon \RR^n \to \RR^n$
\[
k_\varphi(x)_i = x_i(1+\sum_j a_{i,j} x_j).
\]
For any $\ell=1,\dots,n$, we get
\[
\frac{\partial k_i}{\partial x_\ell} = 
\delta_{i,\ell}(1+\sum_j a_{i,j}x_j) + x_i a_{i,\ell}.
\]
The derivative of $k_\varphi$ at $x \in \RR^n$ is the matrix $M$ with entries $\tfrac{\partial h_i}{\partial x_\ell}$ so 
\[
M_{i,\ell} = \left\{
\begin{array}{ll}
1+ \sum_j a_{i,j} x_j + a_{i,\ell}x_i & \text{if $i=\ell$} \\
a_{i,\ell} x_i  & \text{if $i \neq \ell$}.
\end{array}\right.
\]
In matrix form this can be written as $
M = I + \diag(\Adj x) + \diag(x) \Adj$, which --- by inspection --- takes the form of the matrix shown in Figure \ref{fig:matrix}.

Thus, $M$ has the form in Corollary \ref{C:det positive an positive diagonal} and we deduce that $\det(M)>0$ and the diagonal of $M^{-1}$ has positive entries.
In particular $M$ is invertible.
By the inverse function theorem $k$ is invertible in a neighbourhood of $x$ with an inverse $k_\varphi^{-1}$ defined in a neighbourhood $U$ of $y=k_\varphi(x)$ and the derivative of $k_\varphi^{-1}$ at $y$ is equal to $M^{-1}$.

Consider some $w \in W$ and set $x = h_\varphi(w)$.
We have seen that the matrix $M$ of the derivative of $k_\varphi$ at $x$ is invertible and $M^{-1}$ has positive diagonal entries.
Also $k_\varphi$ is invertible at a neighbourhood of $x$ with inverse $k_\varphi^{-1}$ defined in a neighbourhood $U$ of $k_\varphi(x)=w$ and with derivative $M^{-1}$ at $w$.
But $h_\varphi$ is the inverse of $k_\varphi|_{H_\varphi}$ so $h_\varphi$ must coincide with $k_\varphi^{-1}$ on $U \cap W$ and in particular $h_\varphi$ is $C^\infty$ at $w$ and its derivative is given by $M^{-1}$ whose diagonal entries are positive.
\end{proof}

\section{Conclusions and Future Work}\label{sec:conclusion}

Like \cite{oren2022inverse}, this paper primarily examines the inverse problem in gradual argumentation semantics. It proves several conjectures made in that work, namely demonstrating that the $HC, MB$ and $CB$ semantics are continuous, and that the $HC$ semantics are strongly monotonic. In addition, we prove that a more general class of gradual semantics is continuous, and has a unique fixed point, than done in work such as \cite{Pu14,AMGOUD2022103607}. Using these properties we demonstrate that --- for any semantics which can be represented using an argumentation kernel --- we can compute the initial weights analytically, obviating the need for the numerical method described in \cite{oren2022inverse}.

We are pursuing several avenues of future work. First, we want to investigate whether our monotonicity results can be extended to more general argumentation kernels, as well as investigating which existing weighted and ranking semantics (e.g., \cite{dunne_weighted_2011,rago_discontinuity-free_2016,pereira_changing_2011}) can be represented using such argumentation kernels, and whether there is an overlap between the argumentation kernel approach and the properties described in \cite{AMGOUD2022103607,bonzon_comparative_2016}. Second, we would like to further investigate the utility of the class of semantics described by argumentation kernels, and identify whether useful semantics exist which cannot be described by such kernels but which still respect monotonicity and continuity. If such semantics are found, then the we wish to determine whether an analytical solution to the inverse problem for such semantics can be found, or whether a numerical method for solving the inverse problem in such cases is applicable. In addition, characterising the shape of the applicability degree in terms of a semantics and initial weights may yield interesting insights. We want to examine how our work can be applied to dynamic argumentation and sensitivity analysis of arguments. In effect, our work is an initial step to asking by how much an argument must change in order to change some conclusion. However, it does not (as yet) provide a lower bound to this change, but rather only a sufficient bound.  Finally, we observe that the use of argumentation kernel functions have the potential to allow different arguments to compute a final acceptability degree in different ways. This could allow for the modelling of --- for example --- different acceptability degrees arising from different types of arguments, c.f., argumentation schemes, and serves as an exciting avenue of future research.

\bibliographystyle{unsrt}
\bibliography{arxiv.bib}

\end{document}